\documentclass{article}

\usepackage{microtype}
\usepackage{graphicx}
\usepackage{subcaption}
\usepackage{booktabs} %

\usepackage{hyperref}

\usepackage[accepted]{icml2026}

\usepackage{amsmath}
\usepackage{amssymb}
\usepackage{mathtools}
\usepackage{amsthm}
\usepackage{proof}

\usepackage[T1]{fontenc}
\usepackage[utf8]{inputenc}
\usepackage[scaled=0.9]{DejaVuSansMono}    %
\usepackage{listings}

\usepackage{xcolor}
\definecolor{keywordcolor}{HTML}{4169e1}   %
\definecolor{tacticcolor}{HTML}{4169e1}    %
\definecolor{commentcolor}{HTML}{2e8b57}   %
\definecolor{symbolcolor}{HTML}{000000}%
\definecolor{sortcolor}{HTML}{4169e1}      %
\definecolor{attributecolor}{HTML}{f75394} %
\definecolor{bgcolor}{gray}{0.95}

\usepackage[capitalize,noabbrev]{cleveref}

\newcommand{\NN}{\mathbb{N}}
\newcommand{\EE}{\mathbb{E}}
\newcommand{\PP}{\mathbb{P}}
\newcommand{\dd}{\mathrm{d}}
\newcommand{\eps}{\varepsilon}

\newcommand{\calX}{\mathcal{X}}

\newcommand{\calA}{\mathcal{A}}

\newcommand{\calH}{\mathcal{H}}
\newcommand{\calR}{\mathcal{R}}

\DeclareMathOperator*{\argmin}{arg\,min}
\DeclareMathOperator*{\argmax}{arg\,max}
\DeclareMathOperator{\lip}{Lip}

\theoremstyle{plain}
\newtheorem{theorem}{Theorem}[section]
\newtheorem{proposition}{Proposition}[section]
\newtheorem{lemma}{Lemma}[section]
\newtheorem{corollary}{Corollary}[section]

\newtheorem{theorem*}{Theorem}
\newtheorem{lemma*}{Lemma}

\crefname{example}{Example}{Examples}

\theoremstyle{definition}

\crefname{definition}{Definition}{Definitions}

\newtheorem{assumption}{Assumption}[section]
\crefname{assumption}{Assumption}{Assumptions}

\theoremstyle{remark}

\crefname{remark}{Remark}{Remarks}

\usepackage[textsize=tiny]{todonotes}

\icmltitlerunning{Why Agentic Theorem Prover Works: A Statistical Provability Theory of Mathematical Reasoning Models}

\begin{document}

\twocolumn[
  \icmltitle{Why Agentic Theorem Prover Works:\\A Statistical Provability Theory of Mathematical Reasoning Models}

  \icmlsetsymbol{equal}{*}

  \begin{icmlauthorlist}
    \icmlauthor{Sho Sonoda}{ca,riken}
    \icmlauthor{Shunta Akiyama}{ca}
    \icmlauthor{Yuya Uezato}{ca,nii}
  \end{icmlauthorlist}

  \icmlaffiliation{ca}{CyberAgent, Inc.}
  \icmlaffiliation{riken}{RIKEN}
  \icmlaffiliation{nii}{National Institute of Informatics, Japan}

  \icmlcorrespondingauthor{Sho Sonoda}{sho.sonoda@riken.jp}

  \icmlkeywords{agentic theorem proving, Markov decision process, statistic provability}

  \vskip 0.3in
]

\printAffiliationsAndNotice{}  %

\begin{abstract}
Agentic theorem provers combine a reasoning model, retrieval, search, and a proof assistant verifier, yet it remains unclear which components actually improve finite-budget proof success and why they help on real mathematical workloads. We study this question through \emph{statistical provability}: the probability of reaching a verified proof within a budget on a specified stream of theorem instances. We model formal proof search as a finite-horizon reachability MDP with deterministic verifier dynamics, and show that under a faithful state abstraction the optimal success probability coincides with ordinary syntactic provability. We then analyze a simple but practically important pipeline: depth-wise offline action-value regression followed by greedy test-time proving. Our main theorem bounds the provability gap between the learned prover and the optimal prover by an occupancy-weighted sum of uniform action-value errors; in the common uniform-error reading, the leading complexity multiplier is the learned prover's average truncated proof length. The error decomposes into approximation error, geometric coverage of the training distribution, and Monte Carlo label noise, and improves to a fast rate under an action-gap margin condition. The result gives a component-sensitive account of why verifier feedback, retrieval, representation geometry, and proof-shortening mechanisms help on biased theorem workloads, without contradicting classical worst-case hardness.
\end{abstract}

\section{Introduction}

Large language models have improved rapidly at mathematical reasoning,
and recent systems increasingly translate that capability into
\emph{verified} proofs by combining an LLM policy with retrieval, search,
and a proof assistant verifier
\citep{Polu2020gen,zheng2022miniff,yang2023leandojo,jiang2023draft,xin2024deepseekprover,Ren2025deepseekprover-v2,baba2025proveragent,bytedance2025seedprover,varambally2025hilbert,harmonic2025aristotle}.
These agentic theorem provers are empirically effective, but the usual
theoretical vocabulary does not explain their behavior well. Classical
logic and proof theory ask whether a proof exists. Complexity theory asks
how hard proof search can be in the worst case. A deployed prover faces a
different question: \emph{how likely is this particular algorithm to
generate a verifier-accepted proof within a fixed compute budget on the
problems it actually receives?}

This distinction matters. In standard proof theory, provability means the
existence of a finite proof diagram in a formal calculus: for assumptions
\(\Gamma\), a conclusion \(\varphi\), and a proof system \(K\), the
judgement \(\Gamma\vdash_K\varphi\) is true if such a diagram exists.
That is the right logical notion, and we do not replace it. But it is not
enough to explain why an LLM-guided prover works. A proof may exist while
a finite-budget search procedure almost never finds it; conversely, a
learned prover may succeed often on the theorem families it is trained
for even though unrelated worst-case instances remain hard.

The mechanism we study is simple. Real theorem-proving workloads are not
uniform samples from all formal statements of a given size. They are
biased toward library theorems, benchmark families, reused definitions,
and recurring proof patterns. Training traces and verifier interaction
expose this bias: they show which actions tend to make progress on the
states that actually occur. If a learned scorer estimates those local
success probabilities accurately on the high-mass region visited by the
prover, greedy or search-based proving can have high finite-budget
success probability without contradicting undecidability or worst-case
lower bounds.

We call this algorithmic success probability \emph{statistical
provability}. For a problem stream \(q_0\), a verifier-call budget \(B\),
and a prover policy \(\pi\), it is the probability that the prover reaches
a verified proof within \(B\) calls, averaged over \(X_0\sim q_0\) and
over the prover's own randomness. This quantity depends on both the
problem bias and the algorithm. It is therefore the object one needs in
order to ask why an agentic prover works on its intended workload.
Alongside success probability, we track the prover's average truncated
proof length: the expected number of verifier calls made before success
or budget exhaustion. This is directly readable from proof traces and is
the main interpretability handle in the bound: for a fixed local
value-estimation error, shorter average proofs imply smaller end-to-end
loss.

Our analysis formalizes this question with a finite-horizon reachability
MDP. The state is the current proof obligation together with verifier
feedback, the action is the next prover step, and the transition is the
verifier's deterministic response. The Bellman structure of this MDP
turns successful proving into an action-value estimation problem. We
analyze a simple pipeline aligned with current practice: learn depth-wise
action-value functions offline from verifier-backed rollouts, then prove
greedily at test time using the learned scores. The resulting theorem
shows how approximation error, representation geometry, data coverage,
rollout budget, and action-gap margins affect end-to-end proof success.

\paragraph{Contributions.}
\begin{itemize}
\item
  We formalize agentic theorem proving as a finite-horizon reachability
  MDP and define \emph{statistical provability} as the finite-budget
  probability that a specified prover generates a verified proof on
  \(X_0\sim q_0\).
\item
  We show that, under a faithful abstraction of proof states, optimal
  MDP success is exactly syntactic provability within the same verifier
  budget. This keeps the logic intact while enabling value-based
  analysis.
\item
  We analyze \emph{offline action-value regression + greedy proving} and
  derive a provability bound whose leading multiplier is the learned
  prover's average truncated proof length and whose statistical term
  separates approximation error, coverage geometry, and Monte Carlo label
  noise.
\item
  We interpret the theorem in terms of concrete design choices:
  retrieval, verifier feedback, representation learning, and
  proof-shortening mechanisms help by reducing average proof length,
  improving margins, or localizing the learning problem to better
  covered regions.
\end{itemize}

\section{Proof Search as a Reachability MDP}
\label{sec:setup}

\paragraph{Ordinary proof systems.}
Fix a formal proof system \(K\). At the logical level, one can think of
\(K\) as specifying a language of formulas, a set of axioms, and a set of
inference rules. A \(K\)-proof of a conclusion \(\varphi\) from
assumptions \(\Gamma\) is a finite proof tree whose leaves are axioms or
assumptions, whose internal nodes follow the inference rules, and whose
root is \(\varphi\). When such a proof tree exists, we write
\(\Gamma\vdash_K\varphi\). This is the standard proof-theoretic notion of
\emph{provability} used in logic.

\paragraph{Proof assistants and proof-search states.}
A proof assistant implements such a formal system together with a trusted
verifier. Lean \citep{deMouraUllrich2021lean4} is a representative example: its kernel checks proof
objects, while users and agents usually interact through tactics. During
backward proof search, the visible state is a finite list of open goals,
each of the form ``under local context \(\Gamma_i\), prove
\(\varphi_i\).'' A tactic transforms this list into a new list of goals,
and the proof is complete when the list is empty. Thus a proof assistant
already exposes the interface needed for a sequential decision problem:
current proof state, proposed action, verifier-approved next state.
This is not specific to backward search. In forward proof search, the
state can instead be the set of formulas derived so far together with
the target, an action selects an inference-rule instance and its
premises, and the verifier transition adds the derived conclusion.
\Cref{fig:tp-as-mdp} illustrates these two views on the same modus-ponens
proof and labels the corresponding MDP components.

\begin{figure*}[t]
  \centering
  \includegraphics[width=0.8\textwidth,trim=10 90 40 0,clip]{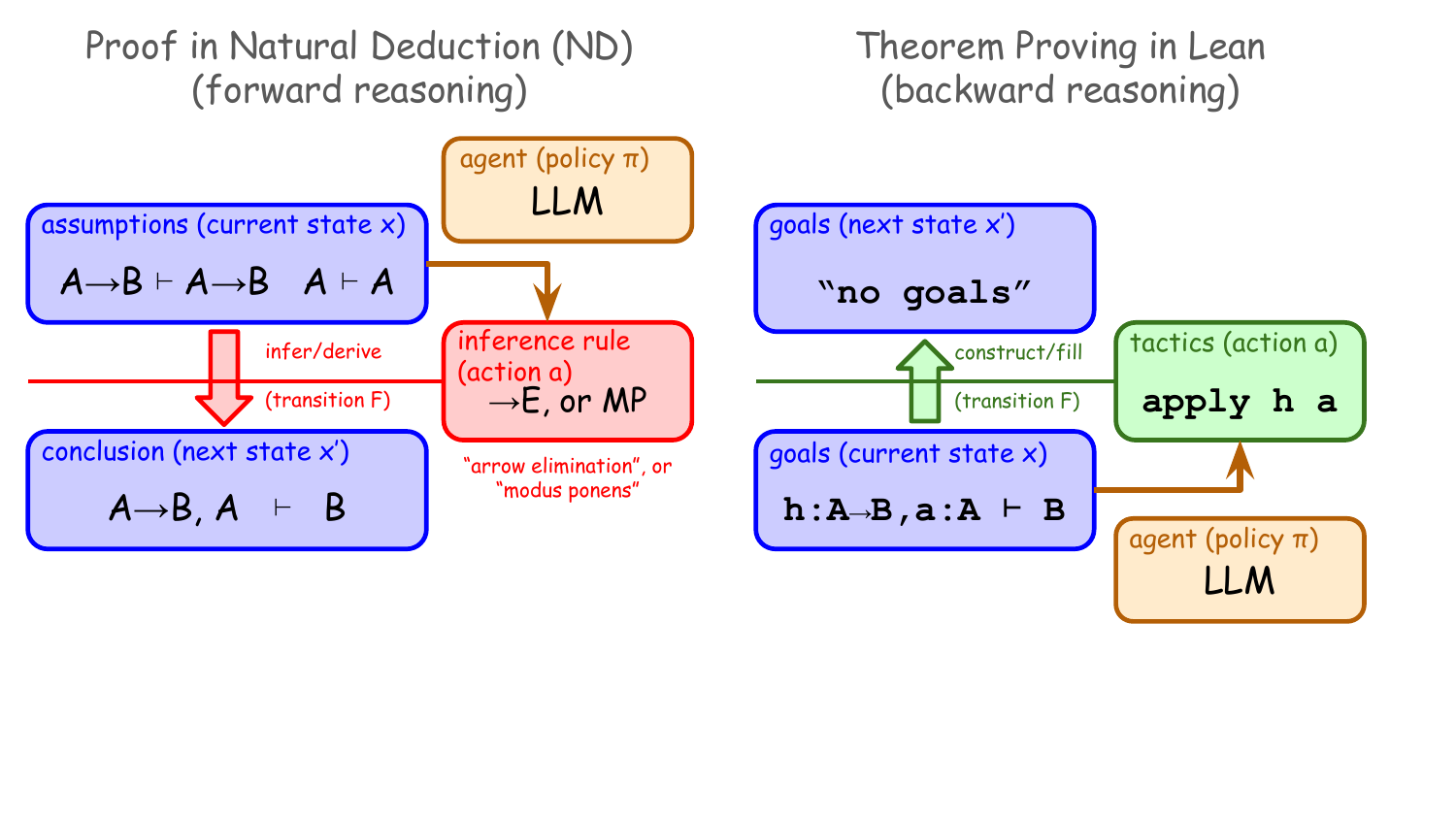}
  \caption{
  Formal proving as a verifier-defined MDP.
  Left: in forward natural-deduction search, the current state \(x\)
  records available assumptions or derived facts; an action \(a\)
  selects an inference-rule instance such as implication elimination
  (modus ponens); and the transition \(F(x,a)=x'\) adds the derived
  conclusion, here \(B\).
  Right: in Lean-style backward search, the current state \(x\) is the
  open goal \(h:A\to B,\ a:A\vdash B\); the tactic action
  \(\texttt{apply h a}\) supplies an exact proof term for that goal;
  and the verifier transition returns the solved state \(x'\) with no
  remaining goals.
  The diagram suppresses many implementation details, but the MDP
  correspondence is always state, action, verifier transition, solved
  set.
  }
  \label{fig:tp-as-mdp}
\end{figure*}

\paragraph{Verifier interface.}
We abstract the interface in \cref{fig:tp-as-mdp} as follows. Let
\(\calX\) be the proof-state representation. A state contains
the currently open obligations together with local context, metavariables,
and verifier messages. Let \(\calA\) be an action space of prover moves
such as tactics, inference-rule instances, lemma invocations, or
retrieval-conditioned commands. We assume a deterministic verifier update
\[
  F:\calX\times\calA\to \calX,
\]
where invalid actions are mapped to explicit failure states. In forward
natural deduction, for example, the action ``apply modus ponens to
\(A\to B\) and \(A\)'' maps a state containing those facts to a state
that also contains \(B\). In Lean tactic mode, tactics construct proof
terms incrementally:
\lstinline{apply} unifies the conclusion of a supplied expression with
the current goal and creates subgoals for missing arguments, while
\lstinline{exact} fills the current goal exactly. Thus a tactic command
is naturally an action whose verifier-approved result is a new proof
state. The same abstraction also covers backward sequent search, Rocq,
and Isabelle.

For a budget \(B\in\NN\), let
\(\mathrm{Prov}_B(x)\in\{0,1\}\) denote whether a closed proof can
be reached from \(x\) within at most \(B\) verifier calls. This is
the budgeted analogue of ordinary proof-theoretic provability. It is
still a pointwise logical property of the initial state and the verifier,
not a statement about whether a learned prover will find the proof.

\begin{proposition}[Verifier interfaces induce deterministic MDPs]
\label{prop:det-mdp}
Any proof-search system specified by a state space
\(\calX\), an action space \(\calA\), a verifier update map
\(F:\calX\times\calA\to\calX\), and
a solved set \(\mathsf G\subseteq \calX\)
induces a deterministic finite-horizon reachability MDP via
\[
  P(\cdot\mid x,a)=\delta_{F(x,a)}.
\]
\end{proposition}

This proposition is mainly a bookkeeping device. The MDP is not a new
logic; it is the verifier's step-by-step operational semantics written
in the notation needed for dynamic programming. Randomness may enter
through stochastic policies, exploration, retrieval, or decoding, but
the verifier transition itself is deterministic in the systems we
target.

\paragraph{Why whole proof states, not single goals?}
One proof step can create several subgoals that must all be discharged.
The induced search object is therefore closer to an AND-OR graph or
directed hypergraph than to a single trajectory. Modeling the state as
the full multiset of currently open goals restores the Markov property:
all future obligations are summarized by the current proof state, and
branching is handled inside the verifier update.

\paragraph{State representations and statistical provability.}
A convenient proof-state representation is a finite measure over embedded
open goals, but the main theorem only needs a compact metric
representation on the region where training and testing concentrate. For
a policy
\(\pi\), define the hitting time
\[
  T_{\mathsf G}:=\inf\{t\ge 0: X_t\in \mathsf G\},
  \qquad X_{t+1}=F(X_t,A_t).
\]
The budgeted success probability is
\[
  V_B^\pi(x):=\PP_x^\pi(T_{\mathsf G}\le B),
  \qquad
  V_B^*(x):=\sup_\pi V_B^\pi(x).
\]
Given an initial-state distribution \(q_0\in\mathcal P(\calX)\), we
define
\[
  \mathrm{SP}_B^\pi(q_0)
  :=\EE_{q_0}\left[V_B^\pi(X)\right], \quad 
  \mathrm{SP}_B^*(q_0)
  :=\EE_{q_0}\left[ V_B^*(X)\right].
\]
We call these quantities \emph{statistical provability}.
The companion length statistic is the average truncated proof length
\begin{equation}
\label{eq:avg-proof-length}
  \begin{aligned}
  \mathrm{Len}_B^\pi(q_0)
  &:=
  \EE_{q_0}^{\pi}[T_{\mathsf G}\wedge B]
  =
  \sum_{s=1}^B \PP_{q_0}^{\pi}(T_{\mathsf G}\ge s).
  \end{aligned}
\end{equation}
where the equality is the tail-sum formula. Operationally,
\(\mathrm{Len}_B^\pi(q_0)\) is the average number of verifier calls
spent by prover \(\pi\) before it either finds a proof or exhausts the
budget. Unlike worst-case proof length, it is estimated from the same
traces used to evaluate a prover, which makes it a useful quantity for
interpreting the theorem.

\paragraph{Learning setup and notation.}
We use roman letters for data-generating objects and Greek letters for
learned policies. The initial theorem distribution is \(q_0\). At
remaining horizon \(t\), the training query pairs are drawn from a
data-query distribution \(q_t\) supported on a relevant compact
state-action region \(\mathcal C_t\). The target of regression is the
optimal one-step value
\[
  Q_t^*(x,a):=V_{t-1}^*(F(x,a)).
\]
The learning side consists of a hypothesis class \(\calH_t\), an
estimator \(\widehat Q_t\in\calH_t\), and the induced greedy policy
\[
  \hat\pi_t(x)\in\argmax_{a:(x,a)\in\mathcal C_t}\widehat Q_t(x,a).
\]
We write \(T_{\mathsf G}\) for the hitting time of the solved set and
\(O_s^{\hat\pi}\) for the conditional occupancy law given non-solution.
The full statistical assumptions on \(q_t,\calH_t,\widehat Q_t\), and the
rollout labels are stated in \Cref{ass:main}.

\begin{proposition}[Faithful abstraction preserves provability]
\label{prop:faithful}
Suppose a lower-level state space \(\calR\) with verifier update
\(F_{\calR}:\calR\times\calA\to\calR\) and solved set
\(\mathsf G_{\calR}\) is represented in \(\calX\) by an encoding map
\(e:\calR\to\calX\). Assume that the encoding is faithful in the sense
that
\[
  e(F_{\calR}(r,a)) = F(e(r),a)
\]
for every state-action pair in the lower-level representation, and that
solved states are preserved:
\[
  r\in\mathsf G_{\calR}
  \quad\Longleftrightarrow\quad
  e(r)\in\mathsf G.
\]
Let \(\mathrm{Prov}_B(r)\) denote reachability of
\(\mathsf G_{\calR}\) from \(r\) within \(B\) lower-level verifier calls.
Then for every \(r\in\calR\),
\[
  V_B^*(e(r))=\mathbf 1[\mathrm{Prov}_B(r)=1].
\]
Consequently, if \(q_0\) is the push-forward of a problem distribution
on \(\calR\), then \(\mathrm{SP}_B^*(q_0)\) is exactly the mass of
instances that are syntactically provable within budget \(B\).
\end{proposition}

\paragraph{Why this matters.}
\Cref{prop:faithful} says the MDP language does not replace the
underlying logic; it repackages it in a form compatible with dynamic
programming and statistical learning. This is the bridge between
classical provability and the finite-budget generation probability that
matters for agentic systems.

\paragraph{Encodings and concentration regions.}
The theorem below only needs a metric representation of the proof states
on the state-action region where training and test-time comparisons take
place. One possible implementation is a measure-valued encoding of the
current multiset of open goals, but the analysis is stated directly in
terms of the concentration sets \(\mathcal C_t\). Details about variable
numbers of goals and high-probability truncation are deferred to
\cref{app:variable-goals}.

\section{Offline Action-Value Learning}
\label{sec:learning}

We now analyze a depth-wise value-learning pipeline. For remaining
horizon \(t=1,\dots,B\), define the optimal one-step action-value
function
\[
  Q_t^*(x,a):=V_{t-1}^*(F(x,a)).
\]
Given estimators \(\widehat Q_t\), the learned prover acts greedily:
\[
  \hat\pi_t(x)\in\argmax_{a:(x,a)\in \mathcal C_t}\widehat Q_t(x,a).
\]
This deliberately isolates the central scoring problem. Retrieval,
decomposition, reranking, and tool feedback matter in our theory only
through how much they improve these action-value estimates or shorten
the remaining proof.

\paragraph{Why action-value learning rather than policy learning?}
The theorem-proving pipelines that motivate this paper are usually not
trained end to end as policies over full proof trees. In practice, many
components act more like \emph{scorers} or \emph{rerankers}: given a
state and a candidate action, they estimate whether this move is likely
to lead to a proof under the remaining budget. Our formulation matches
that reality. Training is offline regression against verifier-backed
success targets, while test-time proving is greedy decoding with respect
to the learned action values.

\begin{assumption}[Depth-wise offline value-learning model]
\label{ass:main}
Fix a budget \(B\) and an initial-state distribution \(q_0\). For each
remaining horizon \(t=1,\dots,B\), assume:
\begin{enumerate}
\item
  There is a compact concentration set
  \(\mathcal C_t\subseteq \calX\times\calA\) containing both the training
  query pairs and the state-action comparisons performed by the learned
  prover. For every relevant state \(x\) at which the regret bound is
  evaluated, the slice
  \(\calA_t(x):=\{a:(x,a)\in\mathcal C_t\}\) is nonempty, the greedy
  action \(\hat\pi_t(x)\in\argmax_{a\in\calA_t(x)}\widehat Q_t(x,a)\)
  exists, and there is an optimal Bellman action
  \(a_t^*(x)\in\calA_t(x)\) such that
  \[
    Q_t^*(x,a_t^*(x))
    =
    \sup_{a\in\calA}Q_t^*(x,a).
  \]
\item
  \(Q_t^*\) is \(L_{Q,t}\)-Lipschitz on \(\mathcal C_t\), and the hypothesis
  class satisfies
  \(\calH_t\subseteq \lip_{L_{H,t}}(\mathcal C_t,[0,1])\).
\item
  Query pairs
  \((X_{t,1},A_{t,1}),\dots,(X_{t,N_t},A_{t,N_t})\) are i.i.d.\
  samples from an exploration distribution \(q_t\) supported on
  \(\mathcal C_t\), and \(q_t\) has local mass lower bound
  \[
    q_t(B(z,r))\ge c_t r^{d_t}
  \]
  for all \(z\in\mathcal C_t\) and sufficiently small \(r>0\), where
  \(B(z,r)\) is the metric ball in \(\mathcal C_t\). Here \(N_t\) is the
  number of sampled state-action query pairs, \(d_t\) is the effective
  local dimension, and \(c_t>0\) is the corresponding coverage constant.
\item
  Each query pair receives \(m_t\) conditionally independent binary
  rollout labels \(Y_{t,i,1},\dots,Y_{t,i,m_t}\in\{0,1\}\) such that
  \[
    \EE[Y_{t,i,j}\mid X_{t,i},A_{t,i}]
    =
    Q_t^*(X_{t,i},A_{t,i}).
  \]
  Thus \(m_t\) is the number of independent verifier-backed rollout
  labels collected for each sampled query pair.
\item
  Writing
  \(\bar Y_{t,i}:=m_t^{-1}\sum_{j=1}^{m_t}Y_{t,i,j}\), the estimator is
  chosen by minimax regression on the sampled points:
  \[
    \widehat Q_t
    \in
    \argmin_{h\in\calH_t}
    \max_{1\le i\le N_t}
    |h(X_{t,i},A_{t,i})-\bar Y_{t,i}|.
  \]
\end{enumerate}
\end{assumption}

\paragraph{Relevant domains and stability.}
The set \(\mathcal C_t\) should be read as the non-negligible
state-action region occupied by the data generator and by the learned
prover at remaining horizon \(t\). Empirically, it can be approximated
by the support, or a high-mass enlargement, of the available prover
traces. The assumption is not automatically stable under arbitrary model
changes: if a new prover leaves this occupied region, the bound no
longer certifies its behavior without additional coverage assumptions.
This is why the theorem is a trace-local diagnostic rather than a global
claim about the entire proof-search space.

\paragraph{Optimizer and retention conditions.}
The optimizer clauses in \Cref{ass:main} have different roles. The
existence of the greedy action is a mild well-posedness condition: once
\(\calA_t(x)\) is nonempty, compactness of the slice and Lipschitz
continuity of \(\widehat Q_t\) give a maximizer by Weierstrass' theorem.
The optimal Bellman-action clause is stronger. It says not only that an
optimal first action exists, but also that at least one such action lies
inside the candidate region \(\mathcal C_t\) on the states where the
learned prover is evaluated. This is a candidate-retention assumption:
the data generator and retrieval mechanism must not exclude every
optimal move from the local comparison set. Standard compact/Feller
conditions that guarantee Bellman maximizers are recalled in
\cref{app:wellposed}; the main theorem only needs their trace-local
consequence stated in \Cref{ass:main}. If one has only an
\(\eta_t\)-optimal retained action, the same proof gives the bound with
an additional occupancy-weighted \(\eta_t\) error term.

\paragraph{Regularity and optimization idealizations.}
The Lipschitz assumptions on \(Q_t^*\) and \(\calH_t\) are representation
assumptions: they assert that, on the occupied state-action region, the
chosen encoding makes nearby proof states have comparable future success
probabilities and makes the learned score class no rougher than
\(L_{H,t}\). The exact minimax regression step is also an idealization.
Allowing a numerical or statistical optimizer whose empirical minimax
loss is within \(\xi_t\) of the infimum simply adds a corresponding
\(\xi_t\) term to the value-estimation error \(\eps_t\).

\paragraph{Rollout targets.}
The binary labels \(Y_{t,i,j}\) are idealized verifier-backed success
labels: after proposing action \(A_{t,i}\) at state \(X_{t,i}\), one
continues for the remaining budget and records whether a verified proof
is reached. The assumption
\(\EE[Y_{t,i,j}\mid X_{t,i},A_{t,i}]=Q_t^*(X_{t,i},A_{t,i})\) is an
oracle version of this construction. In practice, repeated rollouts,
teacher-guided rollouts, or bootstrapped value targets only approximate
this oracle target; the resulting bias is absorbed into the
approximation term in \(\eps_t\).

\begin{theorem}[Offline action-value regression implies statistical provability]
\label{thm:main}
Under \Cref{ass:main}, for any \(\delta_1,\dots,\delta_B\in(0,1)\),
with probability at least \(1-\sum_{t=1}^B\delta_t\) over the sampled
query pairs and rollout labels,
\begin{equation}
\label{eq:main-occ}
  \mathrm{SP}_B^{\hat\pi}(q_0)
  \ge
  \mathrm{SP}_B^*(q_0)
  -
  2\sum_{s=1}^B
  \PP_{q_0}^{\hat\pi}(T_{\mathsf G}\ge s)\,\eps_{B-s+1},
\end{equation}
where the probabilities
\(\PP_{q_0}^{\hat\pi}(T_{\mathsf G}\ge s)\) form the learned prover's
unsolved-mass curve, and where
\begin{align}
\label{eq:main-eps}
  \eps_t
  &:=
  \underbrace{
	    \inf_{h\in\calH_t}\|h-Q_t^*\|_{L_\infty(\mathcal C_t)}
	  }_{\text{approximation}} \notag\\
	  & \quad +
	  \underbrace{
	    C_{\mathrm{cov}}(L_{H,t}+L_{Q,t})c_t^{-1/d_t}
	    \left(\frac{\log(4N_t/\delta_t)}{N_t}\right)^{1/d_t}
	  }_{\text{coverage / geometry}} \notag \\
	  & \quad +
	  \underbrace{
	    2\sqrt{\frac{\log(4N_t/\delta_t)}{2m_t}}
	  }_{\text{rollout noise}}.
\end{align}
Here \(C_{\mathrm{cov}}>0\) is a universal covering constant. The three
terms in \(\eps_t\) are, respectively, approximation error, finite-sample
coverage error over \(\mathcal C_t\), and rollout-label noise.
If in addition the conditional occupancy distribution
\[
  O_s^{\hat\pi}
  :=
  \mathcal L(X_{s-1}\mid T_{\mathsf G}\ge s)
\]
satisfies the margin condition, for every \(s=1,\dots,B\) and
every \(u\ge 0\),
\[
  O_s^{\hat\pi}(\Delta_t\le u)\le C_\Delta u^\beta,
  \qquad t=B-s+1,
\]
for some constants \(C_\Delta>0\) and \(\beta\ge 0\), where
\(\Delta_t(x)\) is the gap between the best and second-best candidate
action values at remaining horizon \(t\), then on the same event,
\begin{equation}
\label{eq:main-margin}
  \begin{aligned}
  \mathrm{SP}_B^{\hat\pi}(q_0)
  &\ge
  \mathrm{SP}_B^*(q_0)
  -
  2C_\Delta\sum_{s=1}^B
  \PP_{q_0}^{\hat\pi}(T_{\mathsf G}\ge s)\\
  &\qquad\cdot
  (2\eps_{B-s+1})^{\beta+1}.
  \end{aligned}
\end{equation}
\end{theorem}

\paragraph{Average-length reading.}
The first bound is immediately interpretable because the probabilities
\(\PP_{q_0}^{\hat\pi}(T_{\mathsf G}\ge s)\) sum to the learned prover's
average truncated proof length in \cref{eq:avg-proof-length}. In
particular, if
\(\sup_t\eps_t\le \bar\eps_B\), then
\[
  \mathrm{SP}_B^*(q_0)-\mathrm{SP}_B^{\hat\pi}(q_0)
  \le
  2\bar\eps_B\,\mathrm{Len}_B^{\hat\pi}(q_0).
\]
Thus the theorem does not only say that smaller value-estimation error is
better; it says where that error is accumulated. A prover that resolves
typical instances quickly has a smaller multiplier, even if the local
error level \(\bar\eps_B\) is unchanged.

\paragraph{Proof ingredients.}
The argument has three layers. First, a deterministic Bellman recursion
expresses the optimal finite-budget success probability in terms of the
one-step action-value functions \(Q_t^*\). Second, a uniform bound on
\(\widehat Q_t-Q_t^*\) yields a one-step greedy regret inequality, and
telescoping this inequality along the learned trajectory produces the
occupancy-sensitive bound in \cref{eq:main-occ}. Third, the regression
analysis shows that \(\widehat Q_t\) is uniformly close to \(Q_t^*\) on the
relevant compact concentration set by combining sample coverage of the
query pairs with concentration of the averaged rollout labels. Under a
margin condition, only near-tie states can incur harmful misranking,
which upgrades the linear dependence on \(\eps_t\) to the fast-rate form
in \cref{eq:main-margin}. \cref{app:proofs} contains the full
statements and proofs.

\paragraph{Budget allocation.}
Let \(n_t:=N_t m_t\) be the total verifier-label budget at depth \(t\).
Write
\[
  \eps_{\mathrm{app},t}
  :=
  \inf_{h\in\calH_t}\|h-Q_t^*\|_{L_\infty(\mathcal C_t)}
\]
for the approximation term in \cref{eq:main-eps}.
Balancing the geometry term and the rollout-noise term gives
\[
  N_t\asymp n_t^{d_t/(d_t+2)},
  \qquad
  m_t\asymp n_t^{2/(d_t+2)},
\]
which yields
\[
  \eps_t
  =
  \eps_{\mathrm{app},t}
  +
  \widetilde O\!\left(
    (L_{H,t}+L_{Q,t})c_t^{-1/d_t}n_t^{-1/(d_t+2)}
  \right).
\]
Thus the statistical rate depends on an \emph{effective geometry}
through \((d_t,c_t)\), not on the raw syntactic size of the logic.

\paragraph{A shortest-proof interpretation.}
Let \(L^*(x)\) denote the shortest horizon \(t\) such that
\(V_t^*(x)=1\), with \(L^*(x)=\infty\) if no budget-feasible proof
exists. Then \(\mathrm{SP}_B^*(q_0)\) is precisely the \(q_0\)-mass of
instances with \(L^*(x)\le B\). In other words, the target being learned
is not arbitrary reward shaping; it is the probability mass of theorem
instances whose shortest verified proofs fit inside the compute budget.
This interpretation is useful when comparing proof-shortening
mechanisms, because decomposition or lemma introduction can improve
statistical provability either by increasing action-value regularity or
simply by reducing the relevant shortest proof lengths.

\begin{corollary}[Interpretation on provable instances]
\label{cor:conditional}
If \(\sup_t\eps_t\le \bar\eps_B\), then on the event of
\Cref{thm:main},
\[
  \mathrm{SP}_B^{\hat\pi}(q_0)
  \ge
  \mathrm{SP}_B^*(q_0)
  -
  2\bar\eps_B\,\mathrm{Len}_B^{\hat\pi}(q_0).
\]
If moreover \(q_0(\mathrm{Prov}_B=1)>0\), then by
\Cref{prop:faithful},
\[
  \EE_{q_0}\!\left[
    V_B^{\hat\pi}(X)\mid \mathrm{Prov}_B(X)=1
  \right]
  \ge
  1-
  \frac{
    2\bar\eps_B\,\mathrm{Len}_B^{\hat\pi}(q_0)
  }{
    q_0(\mathrm{Prov}_B=1)
  }.
\]
\end{corollary}

The corollary makes the comparison especially concrete: the theorem is
not only about abstract value gaps. It lower bounds the learned prover's
average success probability on the subset of instances that are
syntactically provable within the budget.

\section{What the Bound Says About Agentic Provers}
\label{sec:implications}

\paragraph{Average proof length matters through occupancy.}
If \(\sup_t \eps_t\le \bar\eps_B\), then \cref{eq:main-occ} implies
\[
  \begin{aligned}
  \mathrm{SP}_B^*(q_0)-\mathrm{SP}_B^{\hat\pi}(q_0)
  &\le
  2\bar\eps_B \sum_{s=1}^B
  \PP_{q_0}^{\hat\pi}(T_{\mathsf G}\ge s)\\
  &=
  2\bar\eps_B\,\mathrm{Len}_B^{\hat\pi}(q_0).
  \end{aligned}
\]
The regret therefore scales with
\(\mathrm{Len}_B^{\hat\pi}(q_0)\), the learned prover's average truncated
proof length. Any mechanism that shortens proofs or resolves them
earlier, such as useful decomposition or better subgoal selection,
directly reduces the end-to-end loss multiplier.

It is worth stressing that the theorem depends on the
\emph{learned prover's} unsolved-mass curve, not directly on the
distribution of optimal shortest proof lengths. This is unavoidable. A
single early mistake can move the learned prover into a dead end or into
a region where the remaining shortest proof is much longer than under an
optimal policy. Without extra assumptions, the quantity that propagates
through the Bellman argument is therefore the occupancy actually induced
by the learned policy itself.

\paragraph{Retrieval and representation help by reshaping the learning problem.}
The geometry term in \cref{eq:main-eps} separates three effects:
\(L_{Q,t}\) captures how irregular the true action-value is under the
chosen representation, \(L_{H,t}\) captures how restrictive the score
class is, and \(c_t^{-1/d_t}\) captures how well the exploration
distribution covers the region that matters. Retrieval, premise
selection, and representation learning help when they make \(Q_t^*\)
smoother, concentrate the relevant state-action pairs on a lower
dimensional region, or increase action margins by pruning obviously bad
choices. This provides a concrete statistical interpretation of pipeline
engineering choices that are often justified only empirically.

This also suggests a concrete representation-learning objective. A good
state representation should not merely make theorem states easy to
cluster; it should make the relevant action-value functions smoother and
the training distribution thicker on the states that matter at test
time. In that sense, the theorem gives a certificate-oriented criterion
for deciding whether a representation is actually useful for proving.

\paragraph{Bellman certificates.}
The same Bellman inequalities can also be used to build upper and lower
certificates for the optimal success probability \(V_B^*\). This is not
needed for the main theorem, whose role is to compare a learned prover
with the optimal finite-budget prover, so we defer the certificate view
and its representation-learning interpretation to \cref{app:representation}.

\paragraph{Verifier feedback matters twice.}
First, it supplies the rollout labels used for action-value regression.
Second, at test time it prevents the prover from drifting silently into
invalid branches. In our framework, better verifier feedback appears as
smaller approximation error, better coverage of the relevant region, and
larger action gaps. This explains why a strong verifier can improve
scaling even when the base policy itself is unchanged.

\paragraph{Variable numbers of goals do not break the theory.}
Real proof states may generate many subgoals. The main theorem does not
require global compactness of the entire state space; it only needs
compact concentration sets \(\mathcal C_t\) at each remaining horizon. A
standard truncation argument adds an overflow penalty \(\delta_W\) when
one wants to reason on a globally compact state space. We state this
extension in \cref{app:variable-goals}.

\paragraph{Why we focus on greedy proving.}
Beam search, top-\(k\) filtering, rollouts, and backtracking are
important in practice. We focus on greedy execution because it exposes
the clearest connection between value estimation and success
probability. The same Bellman argument extends to richer planners once
one adds candidate-retention or planner-suboptimality assumptions; we
briefly discuss this route in \cref{app:scaling}.

\section{Easy and Hard Instance Regimes}
\label{sec:easy-hard}

The theorem organizes a common empirical observation: some theorem
families become easy for agentic provers very quickly, while others
remain stubbornly difficult even when the base reasoning model looks
strong.

\paragraph{Easier instances.}
The favorable regime is characterized by small
\(\mathrm{Len}_B^{\hat\pi}(q_0)\), large action gaps, low-dimensional
concentration of reachable proof states, and good coverage by the
exploration distribution. In that regime,
\(\eps_t\) falls quickly with the verifier-label budget and the
occupancy weights decay rapidly because many trajectories solve early.
This is the setting where retrieval and decomposition are most useful:
they can expose the right lemma early enough that the rest of the proof
becomes almost deterministic.

\paragraph{Harder instances.}
The difficult regime is characterized by long proof horizons, many
near-tie actions, poor training coverage, or proof states whose useful
continuations depend on brittle symbolic details. Then both parts of the
bound become unfavorable: the value-learning problem is statistically
harder, and the occupancy weights stay large because errors made early
can send the prover into dead ends. This is exactly where worst-case
hardness and practical failure modes align.

Using \(\bar\eps_B\) for a uniform bound on the depth-wise errors, the
main theorem can be read schematically as
\[
  \mathrm{SP}_B^*(q_0)-\mathrm{SP}_B^{\hat\pi}(q_0)
  \lesssim
  \mathrm{Len}_B^{\hat\pi}(q_0)\,\bar\eps_B
\]
without a margin condition, and as
\[
  \mathrm{SP}_B^*(q_0)-\mathrm{SP}_B^{\hat\pi}(q_0)
  \lesssim
  \mathrm{Len}_B^{\hat\pi}(q_0)\,\bar\eps_B^{\beta+1}
\]
under a margin condition. This captures the main engineering message
using the same notation as the theorem: to improve proving success, one
can reduce average proof length, reduce value-estimation error, or make
correct actions more separated.

\paragraph{Scaling-law viewpoint.}
The theorem also yields immediate sample-complexity heuristics. Ignoring
logarithmic factors and approximation error, the uniform regression
bound behaves like
\[
  \eps_t \approx n_t^{-1/(d_t+2)}.
\]
Here \(n_t=N_t m_t\) is the verifier-label budget at remaining horizon
\(t\), and \(d_t\) is the effective local dimension from
\Cref{ass:main}. If these quantities are roughly constant across
horizons, say \(n_t\approx n\) and \(d_t\approx d\), then
\cref{eq:main-occ} gives
\[
  \mathrm{SP}_B^*(q_0)-\mathrm{SP}_B^{\hat\pi}(q_0)
  \lesssim
  \mathrm{Len}_B^{\hat\pi}(q_0)\,
  n^{-1/(d+2)}.
\]
Under a margin condition with exponent \(\beta\), the dependence
improves schematically to
\[
  \mathrm{SP}_B^*(q_0)-\mathrm{SP}_B^{\hat\pi}(q_0)
  \lesssim
  \mathrm{Len}_B^{\hat\pi}(q_0)\,
  n^{-(\beta+1)/(d+2)}.
\]
These formulas clarify what kinds of apparent ``test-time scaling''
gains are actually plausible. Large practical improvements do not
require changing worst-case complexity classes; they can arise whenever
retrieval, decomposition, or representation learning reduce the
effective horizon, lower the effective geometric dimension, or increase
the action-gap margin.

This is also why the choice of \(q_0\) matters. The same logical system
may look intractable when \(q_0\) is uniform or adversarial over hard
instances, and much easier when \(q_0\) concentrates on recurring
definitions, library idioms, and proof patterns. Our framework does not
deny the former; it tries to explain the latter.

\section{Related Work}
\label{sec:related}

\paragraph{From symbolic automation to learning-guided search.}
Learning-guided theorem proving predates modern LLMs. TacticToe
\citep{Gauthier2017tactictoe,Gauthier2021} learned tactic guidance and
combined it with Monte Carlo tree search in HOL4. HOList and related
systems \citep{Bansal2019holist} established large-scale environments
for higher-order theorem proving. LeanDojo \citep{yang2023leandojo}
lowered the barrier to training and evaluating Lean-based agents, and
miniF2F \citep{zheng2022miniff} helped standardize formal mathematics
benchmarks. These works already support the sequential-decision view:
proof search unfolds through adaptive interaction with a verifier under
tight computational budgets.
The MDP perspective is also explicit in Bourbaki
\citep{Zimmer2025}, which studies self-generated, goal-conditioned MDPs
and MCTS-like search for theorem proving. Our use of MDPs has a
different purpose: the MDP is the mathematical interface through which
we analyze finite-budget provability, Bellman certificates, and
statistical score-estimation errors, rather than an implemented search
framework.

\paragraph{Neural language models and formal proving.}
GPT-f \citep{Polu2020gen} demonstrated early on that generative language
models can discover useful formal proofs. The main methodological shift
was from hand-engineered proof features to learned representations of
states, tactics, and libraries. This made it natural to combine
proposal, verification, and repair in multi-step loops rather than
treating theorem proving as one-shot next-step classification.

\paragraph{Agentic provers and decomposition.}
Recent systems increasingly operate as explicit agentic workflows with
proposal, retrieval, decomposition, verification, and refinement. DSP
\citep{jiang2023draft} uses informal proof sketches to guide formal
search. DeepSeek-Prover and DeepSeek-Prover-V2
\citep{xin2024deepseekprover,Ren2025deepseekprover-v2} emphasize data
generation and recursive subgoal decomposition in Lean~4. Prover Agent
\citep{baba2025proveragent}, Seed-Prover
\citep{bytedance2025seedprover}, Hilbert
\citep{varambally2025hilbert}, and Aristotle
\citep{harmonic2025aristotle} all make intermediate lemmas or subgoals
first-class objects. This decomposition trend also echoes classical
proof complexity: useful cuts can dramatically shorten proofs, while
their elimination can cause severe blow-up
\citep{boolos1984dont-eliminate-cut}. Our theory gives a statistical
account of the same phenomenon through effective horizon and action-gap
effects.
This is complementary to \citet{Sonoda2026expatp}, which studies a more
specialized flat versus hierarchical comparison and proves an
exponential sample-complexity separation under cut-aware structure. The
present paper instead develops a general value-based theory of
time-bounded statistical provability and score-guided proving.

Seen from this angle, the current wave of agentic theorem proving is not
just ``LLMs applied to theorem proving.'' It is a convergence of three
older ideas: proof search as sequential decision making, proof
complexity as sensitivity to intermediate lemmas, and inference-time
computation as a resource that can be adaptively allocated.

\paragraph{Test-time scaling, verifiers, and interactive learning.}
Our viewpoint is also related to work on test-time scaling and
verifier-guided inference. Chain-of-thought prompting and theoretical
analyses of reasoning traces show that inference-time computation can
change what a model effectively computes
\citep{kojima2022cot,feng2023cot-theory,phan2023cot}. Recent studies of
adaptive test-time strategies, best-of-\(N\), and verifier-guided
selection show that verifier quality can alter scaling laws
\citep{wu2024scaling,beirami2025bon,setlur2025tts-verifier,botta2025bon}.
Worst-case versus average-case guarantees for LLMs have also been
studied through verifier-based notions such as verifiability
\citep{Amit2026}. That work analyzes learning procedures for
self-proving or verifier-backed models. Our focus is narrower and more
theorem-proving-specific: given a finite-horizon proof-search process,
we expose how action-value error, occupied-region geometry, margins, and
average proof length enter the success probability.
More abstractly, interactive learning theory has long treated queries
and feedback as statistical resources, as in Angluin's \(L^\ast\)
algorithm \citep{angluin1987Lstar}, and recent work on computationally
bounded information and formal-language expressivity provides additional
vocabulary for reasoning systems
\citep{finzi2026epiplexity,strobl2024formal-survey}. What is missing
from that literature is a theorem-proving-specific bridge from
verifier-backed score estimation to end-to-end proof success on
structured instance distributions. That is the gap addressed here.
\cref{app:literature} gives a longer overview.

This interactive-learning connection is more than an analogy. In formal
theorem proving, the verifier is an information channel: each tactic
attempt returns structured feedback about which branch remains open,
which constraints are unsatisfied, and whether an intermediate step was
valid. Learning to prove is therefore not only a problem of fitting a
static predictor from past data. It is also a problem of learning how to
extract useful information from verifier interaction under a limited
budget of proof attempts.

\paragraph{Positioning.}
Our contribution is therefore not a survey of theorem proving systems
per se, nor a generic theory of test-time scaling. It is a bridge
between them: a value-based framework in which proof-theoretic
structure, verifier interaction, representation geometry, and learning
guarantees on the occupied trace region all appear in the same
performance bound.

\section{Conclusion}

We proposed a statistical provability theory for agentic theorem
proving centered on a simple pipeline: offline action-value regression
followed by greedy proving. The theory keeps the logical notion of
provability intact through faithful abstraction, but recasts the
evaluation problem in terms of finite-budget success probability on the
problem stream \(q_0\). The resulting bound is explicit about what helps:
shorter effective proofs, smoother value functions, better covered
representations, lower-variance verifier labels, and larger action
margins. A useful feature of the result is that one complexity multiplier
is directly interpretable from traces: the learned prover's average
truncated proof length. This does not remove worst-case hardness. It
explains why agentic theorem provers can still work well on the
structured theorem distributions that arise in practice.

\paragraph{Limitations.}
The model is intentionally simplified. We analyze greedy proving after
offline value regression, while practical systems use beam search,
backtracking, retrieval refreshes, proof repair, and multiple interacting
agents. The rollout-label assumption is also idealized, and the compact
concentration, coverage, Lipschitz, and margin conditions are modeling
assumptions about the occupied trace region rather than guarantees about
the full proof-search space.

\paragraph{Future work.}
A natural next step is to estimate the quantities in the bound from real
prover traces: occupied-region coverage, effective dimension, label
variance, action-gap tails, and the unsolved-mass curve. Another
direction is to extend the same Bellman-certificate analysis to richer
planners such as beam search, top-\(k\) search with backtracking, and
decomposition policies that explicitly synthesize intermediate lemmas.

\section*{Acknowledgements}
SS was supported by 
JST BOOST JPMJBY24E2, JST CREST JPMJCR25I5,
and JSPS KAKENHI 24K21316.
YU was supported by 
JST CREST JPMJCR21M3.

\section*{Impact Statement}

This paper presents theoretical work whose goal is to advance the field of Machine Learning by developing a statistical framework for analyzing agentic theorem-proving pipelines. The results are not directly deployable as a system and do not, by themselves, cause immediate harm. However, because our analysis characterizes which pipeline components and design choices most strongly affect success probability, it could be repurposed to improve the effectiveness of agentic models in other domains. In particular, a malicious actor could apply the same principles to engineer more capable harmful agents (e.g., by optimizing search, decomposition, or tool use). We therefore emphasize that our intent is to support transparency and safety-oriented understanding of agentic systems.

\bibliography{libraryS}
\bibliographystyle{icml2026}

\newpage
\appendix
\onecolumn
\crefalias{section}{appendix} %
\section{Additional Technical Details}
\label{app:proofs}

This appendix gives full proofs for
\cref{prop:det-mdp,prop:faithful,thm:main,cor:conditional}, records
sufficient conditions for the optimizer clauses in \Cref{ass:main}, and
supports \cref{sec:setup,sec:learning}.

\subsection{Proof of the deterministic-MDP proposition}

\begin{proof}[Proof of \Cref{prop:det-mdp}]
Define the MDP state space to be \(\calX\), the action
space to be \(\calA\), the transition kernel to be
\[
  P(\cdot\mid x,a)=\delta_{F(x,a)},
\]
and the goal set to be \(\mathsf G\). Because the next state
depends only on the current state and chosen action, the Markov property
holds. Because the kernel is a Dirac mass, the transition is
deterministic. Reaching \(\mathsf G\) within \(B\) steps is
exactly the same event as completing the proof within \(B\) verifier
calls, so this yields a deterministic finite-horizon reachability MDP.
\end{proof}

\subsection{Proof of the faithful-abstraction proposition}

\begin{proof}[Proof of \Cref{prop:faithful}]
If \(\mathrm{Prov}_B(r)=1\), then by definition there exists a
valid proof-search trace or certificate that reaches a closed state
within at most \(B\) verifier calls. Following the corresponding action
sequence in the represented system reaches \(\mathsf G\) within the same
budget because the transition and solved-set abstractions are faithful.
Hence
\(V_B^*(e(r))=1\).

Conversely, if \(V_B^*(e(r))=1\), then some action sequence
reaches \(\mathsf G\) in the represented system within \(B\) steps.
Faithful abstraction lifts that sequence back to the lower-level verifier
dynamics and solved set,
yielding a valid proof-search trace and hence a proof certificate within
budget \(B\). Therefore
\(\mathrm{Prov}_B(r)=1\).
\end{proof}

\subsection{Well-posedness of Bellman and greedy maximizers}
\label{app:wellposed}

\begin{proposition}[Sufficient conditions for maximizers]
\label{prop:app-wellposed}
Suppose \(\calX\) and \(\calA\) are compact metric spaces,
\(\mathsf G\subseteq\calX\) is closed, and the deterministic verifier
transition \(F:\calX\times\calA\to\calX\) is continuous. Then for every
finite horizon \(t\), the Bellman value \(V_t^*\) is upper
semicontinuous and, for every state \(x\), the supremum
\[
  \sup_{a\in\calA} Q_t^*(x,a)
\]
is attained by at least one action. Moreover, if
\(\mathcal C_t\subseteq\calX\times\calA\) is compact, the slice
\(\calA_t(x):=\{a:(x,a)\in\mathcal C_t\}\) is nonempty, and
\(\widehat Q_t\) is continuous on \(\mathcal C_t\), then the greedy
maximizer
\[
  \hat\pi_t(x)\in\argmax_{a\in\calA_t(x)}\widehat Q_t(x,a)
\]
exists. Consequently, the optimizer clauses in \Cref{ass:main} hold
whenever the compact concentration set additionally retains at least one
global Bellman maximizer at each relevant state.
\end{proposition}

\begin{proof}
For the Bellman maximizer, use backward induction. The base function
\(V_0^*(x)=\mathbf 1[x\in\mathsf G]\) is upper semicontinuous because
\(\mathsf G\) is closed. If \(V_{t-1}^*\) is upper semicontinuous, then
\[
  Q_t^*(x,a)=V_{t-1}^*(F(x,a))
\]
is upper semicontinuous on \(\calX\times\calA\), since \(F\) is
continuous. An upper semicontinuous function attains its maximum on a
compact set, so \(\sup_{a\in\calA}Q_t^*(x,a)\) is attained. The map
\(x\mapsto\sup_{a\in\calA}Q_t^*(x,a)\) is also upper semicontinuous by
the maximum theorem, and hence
\[
  V_t^*(x)
  =
  \mathbf 1[x\in\mathsf G]\vee \sup_{a\in\calA}Q_t^*(x,a)
\]
is upper semicontinuous.

For the learned greedy action, the slice
\(\{x\}\times\calA\) is closed in \(\calX\times\calA\), so
\(\mathcal C_t\cap(\{x\}\times\calA)\) is compact. Its projection
\(\calA_t(x)\) is compact and nonempty by assumption. Since
\(\widehat Q_t\) is continuous on \(\mathcal C_t\), Weierstrass' theorem
gives a maximizer on this slice. The final claim is exactly the
candidate-retention condition: one of the global Bellman maximizers must
belong to \(\calA_t(x)\).
\end{proof}

\subsection{Bellman structure and occupancy-sensitive regret}

\begin{proposition}[Deterministic Bellman recursion]
\label{prop:app-bellman}
For every remaining horizon \(t\ge 1\),
\[
  V_0^*(x)=\mathbf 1[x\in\mathsf G],
  \qquad
  V_t^*(x)
  =
  \mathbf 1[x\in\mathsf G]
  \vee
  \sup_{a\in\calA} Q_t^*(x,a).
\]
\end{proposition}

\begin{proof}
If \(x\in\mathsf G\), then success has already been achieved and
\(V_t^*(x)=1\) for every \(t\ge 0\). This gives the formula for solved
states and, in particular, for \(V_0^*\). Now fix \(t\ge 1\) and
\(x\notin\mathsf G\). For any policy, the first action \(a\) leads to
the deterministic next state \(F(x,a)\), from which the highest
remaining success probability is \(V_{t-1}^*(F(x,a))=Q_t^*(x,a)\).
Hence
\[
  V_t^*(x)\le \sup_{a\in\calA} Q_t^*(x,a).
\]
Conversely, for any \(\xi>0\) one can choose an action \(a_\xi\) such
that
\[
  Q_t^*(x,a_\xi)\ge \sup_{a\in\calA}Q_t^*(x,a)-\xi.
\]
Taking \(a_\xi\) first and then an optimal continuation policy for the
remaining \(t-1\) steps yields
\[
  V_t^*(x)\ge Q_t^*(x,a_\xi)\ge \sup_{a\in\calA}Q_t^*(x,a)-\xi.
\]
Sending \(\xi\downarrow 0\) proves the claim.
\end{proof}

\begin{lemma}[One-step greedy regret]
\label{lem:app-onestep}
Fix \(t\in\{1,\dots,B\}\) and a relevant state \(x\). If
\[
  \sup_{a:(x,a)\in\mathcal C_t}
  |\widehat Q_t(x,a)-Q_t^*(x,a)|
  \le \eps_t,
\]
then
\[
  \sup_{a:(x,a)\in\mathcal C_t}Q_t^*(x,a)
  -
  Q_t^*(x,\hat\pi_t(x))
  \le 2\eps_t.
\]
\end{lemma}

\begin{proof}
Fix any \(a\) with \((x,a)\in\mathcal C_t\). By greedy selection,
\[
  \widehat Q_t(x,a)\le \widehat Q_t(x,\hat\pi_t(x)).
\]
Therefore
\begin{align*}
  Q_t^*(x,a)
  &\le \widehat Q_t(x,a)+\eps_t \\
  &\le \widehat Q_t(x,\hat\pi_t(x))+\eps_t \\
  &\le Q_t^*(x,\hat\pi_t(x))+2\eps_t.
\end{align*}
Taking the supremum over \(a\) yields the result.
\end{proof}

\begin{theorem}[Occupancy-sensitive provability bound]
\label{thm:app-occupancy}
Assume that for every relevant state \(x\) and every remaining horizon
\(t=1,\dots,B\),
\[
  \sup_{a:(x,a)\in\mathcal C_t}
  |\widehat Q_t(x,a)-Q_t^*(x,a)|
  \le \eps_t.
\]
Assume also the candidate-retention condition that an action
\(a_t^*(x)\) exists with \((x,a_t^*(x))\in\mathcal C_t\) and
\[
  Q_t^*(x,a_t^*(x))
  =
  \sup_{a\in\calA}Q_t^*(x,a)
\]
on those relevant states. Then
\[
  \mathrm{SP}_B^{\hat\pi}(q_0)
  \ge
  \mathrm{SP}_B^*(q_0)
  -
  2\sum_{s=1}^B
  \PP_{q_0}^{\hat\pi}(T_{\mathsf G}\ge s)\,\eps_{B-s+1}.
\]
\end{theorem}

\begin{proof}
For \(t\in\{0,\dots,B\}\), define
\[
  D_t(x):=V_t^*(x)-V_t^{\hat\pi}(x).
\]
If \(x\in\mathsf G\), then \(D_t(x)=0\). Now fix \(t\ge 1\) and
\(x\notin\mathsf G\). By \Cref{prop:app-bellman},
\[
  V_t^*(x)=\sup_{a:(x,a)\in\mathcal C_t}Q_t^*(x,a),
\]
where we used the candidate-retention condition in the theorem
statement.
Because the learned prover chooses \(\hat\pi_t(x)\) first and then
continues with the same policy,
\[
  V_t^{\hat\pi}(x)=V_{t-1}^{\hat\pi}(F(x,\hat\pi_t(x))).
\]
Since \(Q_t^*(x,\hat\pi_t(x))=V_{t-1}^*(F(x,\hat\pi_t(x)))\), we obtain
\[
  D_t(x)
  =
  \sup_{a\in\calA}Q_t^*(x,a)
  -
  Q_t^*(x,\hat\pi_t(x))
  +
  D_{t-1}(F(x,\hat\pi_t(x))).
\]
Applying \Cref{lem:app-onestep} gives
\[
  D_t(x)\le 2\eps_t + D_{t-1}(F(x,\hat\pi_t(x))).
\]
Since \(D_t(x)=0\) on solved states, we may write uniformly
\[
  D_t(x)\le 2\eps_t\,\mathbf 1[x\notin\mathsf G]
  + D_{t-1}(F(x,\hat\pi_t(x))).
\]

Now run the learned prover from \(X_0\sim q_0\), and let \(X_s\) be the
state after \(s\) executed steps. Applying the inequality above to
\(X_{s-1}\) with remaining horizon \(B-s+1\) gives
\[
  D_{B-s+1}(X_{s-1})
  \le
  2\eps_{B-s+1}\mathbf 1[X_{s-1}\notin\mathsf G]
  +
  D_{B-s}(X_s).
\]
Taking expectations and using
\[
  \PP(X_{s-1}\notin\mathsf G)
  =
  \PP_{q_0}^{\hat\pi}(T_{\mathsf G}\ge s),
\]
we obtain
\[
  \EE[D_{B-s+1}(X_{s-1})]
  \le
  2\eps_{B-s+1}
  \PP_{q_0}^{\hat\pi}(T_{\mathsf G}\ge s)
  +
  \EE[D_{B-s}(X_s)].
\]
Summing this recursion over \(s=1,\dots,B\) yields
\[
  \EE[D_B(X_0)]
  \le
  2\sum_{s=1}^B
  \PP_{q_0}^{\hat\pi}(T_{\mathsf G}\ge s)\,\eps_{B-s+1}.
\]
Since \(\EE[D_B(X_0)] = \mathrm{SP}_B^*(q_0)-\mathrm{SP}_B^{\hat\pi}(q_0)\),
the result follows.
\end{proof}

\subsection{Uniform regression error on compact concentration sets}

\begin{lemma}[Sample coverage from the exploration distribution]
\label{lem:app-coverage}
Fix \(t\in\{1,\dots,B\}\). Assume that \(\mathcal C_t\) is compact and
\[
  q_t(B(z,r))\ge c_t r^{d_t}
\]
for all \(z\in \mathcal C_t\) and sufficiently small \(r>0\). Then
there exists a universal constant \(C_{\mathrm{cov}}>0\) such that, with
probability at least \(1-\delta_t\), the sampled query pairs
\(\{(X_{t,i},A_{t,i})\}_{i=1}^{N_t}\) form an \(\eta_t\)-net of
\(\mathcal C_t\), where
\[
  \eta_t
  :=
  C_{\mathrm{cov}} c_t^{-1/d_t}
  \left(\frac{\log(2N_t/\delta_t)}{N_t}\right)^{1/d_t}.
\]
\end{lemma}

\begin{proof}
Fix \(\eta>0\) small enough that the lower-mass condition applies at
scales \(\eta/4\) and \(\eta/2\). Let \(\{z_1,\dots,z_M\}\subseteq
\mathcal C_t\) be a maximal \(\eta/2\)-separated set. Then the balls
\(B(z_j,\eta/4)\) are pairwise disjoint, and each has \(q_t\)-mass at
least \(c_t(\eta/4)^{d_t}\). Since \(q_t\) is a probability measure,
\[
  1
  \ge
  \sum_{j=1}^M q_t(B(z_j,\eta/4))
  \ge
  M\,c_t(\eta/4)^{d_t},
\]
so
\[
  M\le 4^{d_t}c_t^{-1}\eta^{-d_t}.
\]
Maximality implies that the balls \(B(z_j,\eta/2)\) cover
\(\mathcal C_t\).

For any fixed \(j\), the probability that none of the \(N_t\) i.i.d.\
samples falls into \(B(z_j,\eta/2)\) is at most
\[
  (1-q_t(B(z_j,\eta/2)))^{N_t}
  \le
  \exp\bigl(-N_t c_t(\eta/2)^{d_t}\bigr).
\]
By a union bound over \(j=1,\dots,M\), the probability that some
covering ball is missed is at most
\[
  M\exp\bigl(-N_t c_t(\eta/2)^{d_t}\bigr)
  \le
  4^{d_t}c_t^{-1}\eta^{-d_t}
  \exp\bigl(-N_t c_t(\eta/2)^{d_t}\bigr).
\]
Choosing
\[
  \eta
  =
  C_{\mathrm{cov}} c_t^{-1/d_t}
  \left(\frac{\log(2N_t/\delta_t)}{N_t}\right)^{1/d_t}
\]
with \(C_{\mathrm{cov}}\) sufficiently large makes this upper bound at
most \(\delta_t\). On the resulting event, every point of
\(\mathcal C_t\) lies within distance \(\eta\) of some sampled query
pair.
\end{proof}

\begin{theorem}[Uniform regression error]
\label{thm:app-regression}
Fix \(t\in\{1,\dots,B\}\). Under the depth-wise conditions in
\Cref{ass:main}, for every \(\delta_t\in(0,1)\), with probability at
least \(1-\delta_t\),
\[
  \|\widehat Q_t-Q_t^*\|_{L_\infty(\mathcal C_t)}
  \le
	  \eps_{\mathrm{app},t}
	  +
	  C_{\mathrm{cov}}(L_{H,t}+L_{Q,t})c_t^{-1/d_t}
	  \left(\frac{\log(4N_t/\delta_t)}{N_t}\right)^{1/d_t}
	  +
	  2\sqrt{\frac{\log(4N_t/\delta_t)}{2m_t}}.
\]
\end{theorem}

\begin{proof}
Define
\[
  \eps_{\mathrm{app},t}
  :=
  \inf_{h\in\calH_t}\|h-Q_t^*\|_{L_\infty(\mathcal C_t)}.
\]
By \Cref{lem:app-coverage}, with probability at least \(1-\delta_t/2\),
the sampled query pairs form an \(\eta_t\)-net of \(\mathcal C_t\) with
\[
  \eta_t
  =
  C_{\mathrm{cov}} c_t^{-1/d_t}
  \left(\frac{\log(4N_t/\delta_t)}{N_t}\right)^{1/d_t}.
\]
Also, conditional on the sampled query pairs, Hoeffding's inequality
implies that for any \(\alpha_t>0\),
\[
  \PP\!\left(
    |\bar Y_{t,i}-Q_t^*(X_{t,i},A_{t,i})|>\alpha_t
    \,\middle|\,
    X_{t,i},A_{t,i}
  \right)
  \le
  2e^{-2m_t\alpha_t^2}.
\]
Setting
\[
  \alpha_t:=\sqrt{\frac{\log(4N_t/\delta_t)}{2m_t}}
\]
and taking a union bound over \(i=1,\dots,N_t\), we get an event of
probability at least \(1-\delta_t/2\) on which
\[
  \max_{1\le i\le N_t}
  |\bar Y_{t,i}-Q_t^*(X_{t,i},A_{t,i})|
  \le \alpha_t.
\]
Work on the intersection of these two events.

Fix \(\xi>0\). By definition of \(\eps_{\mathrm{app},t}\), there exists
\(h_t^\xi\in\calH_t\) satisfying
\[
  \|h_t^\xi-Q_t^*\|_{L_\infty(\mathcal C_t)}
  \le \eps_{\mathrm{app},t}+\xi.
\]
Since \(\widehat Q_t\) minimizes the maximum empirical absolute deviation,
\begin{align*}
  \max_{1\le i\le N_t}
  |\widehat Q_t(X_{t,i},A_{t,i})-\bar Y_{t,i}|
  &\le
  \max_{1\le i\le N_t}
  |h_t^\xi(X_{t,i},A_{t,i})-\bar Y_{t,i}| \\
  &\le
  \eps_{\mathrm{app},t}+\xi+\alpha_t.
\end{align*}
Therefore, for each sample point,
\begin{align*}
  |\widehat Q_t(X_{t,i},A_{t,i})-Q_t^*(X_{t,i},A_{t,i})|
  &\le
  |\widehat Q_t(X_{t,i},A_{t,i})-\bar Y_{t,i}|
   +
  |\bar Y_{t,i}-Q_t^*(X_{t,i},A_{t,i})| \\
  &\le
  \eps_{\mathrm{app},t}+\xi+2\alpha_t.
\end{align*}

Now fix any \((x,a)\in\mathcal C_t\). Choose a sampled query pair
\((X_{t,i},A_{t,i})\) within distance \(\eta_t\). Since both
\(\widehat Q_t\in \calH_t\) and \(Q_t^*\) are Lipschitz on \(\mathcal C_t\),
\begin{align*}
  |\widehat Q_t(x,a)-Q_t^*(x,a)|
  &\le
  |\widehat Q_t(x,a)-\widehat Q_t(X_{t,i},A_{t,i})|
  +
  |\widehat Q_t(X_{t,i},A_{t,i})-Q_t^*(X_{t,i},A_{t,i})|
   +
  |Q_t^*(X_{t,i},A_{t,i})-Q_t^*(x,a)| \\
  &\le
  (L_{H,t}+L_{Q,t})\eta_t
  +
  \eps_{\mathrm{app},t}+\xi+2\alpha_t.
\end{align*}
Taking the supremum over \((x,a)\in\mathcal C_t\) and then sending
\(\xi\downarrow 0\) proves the claim.
\end{proof}

\subsection{Margin theorem and full proof of the main theorem}

\begin{theorem}[Fast rate under a margin condition]
\label{thm:app-margin}
Assume that for every remaining horizon \(t\) and every relevant state
\(x\),
\[
  \sup_{a:(x,a)\in\mathcal C_t}
  |\widehat Q_t(x,a)-Q_t^*(x,a)|
  \le \eps_t
\]
and that the candidate-retention condition in \Cref{thm:app-occupancy}
holds. Let \(\Delta_t(x)\) be the gap between the best and second-best
candidate action values at remaining horizon \(t\). If the conditional
occupancy distributions \(O_s^{\hat\pi}\) satisfy
\[
  O_s^{\hat\pi}(\Delta_t\le u)\le C_\Delta u^\beta,
  \qquad t=B-s+1,
\]
then
\[
  \mathrm{SP}_B^{\hat\pi}(q_0)
  \ge
  \mathrm{SP}_B^*(q_0)
  -
  2C_\Delta\sum_{s=1}^B
  \PP_{q_0}^{\hat\pi}(T_{\mathsf G}\ge s)
  (2\eps_{B-s+1})^{\beta+1}.
\]
\end{theorem}

\begin{proof}
For elapsed step \(s\in\{1,\dots,B\}\), let \(t=B-s+1\) and define
\[
  r_s(x)
  :=
  \max_{a:(x,a)\in\mathcal C_t} Q_t^*(x,a)
  -
  Q_t^*(x,\hat\pi_t(x)).
\]
By \Cref{lem:app-onestep}, \(r_s(x)\le 2\eps_t\). Moreover, if
\(\Delta_t(x)>2\eps_t\), then the perturbation \(\widehat Q_t-Q_t^*\) is too
small to change the identity of the optimal action, so \(r_s(x)=0\).
Hence
\[
  r_s(x)\le 2\eps_t\,\mathbf 1[\Delta_t(x)\le 2\eps_t].
\]
Conditioning on \(T_{\mathsf G}\ge s\), the state \(X_{s-1}\) has
distribution \(O_s^{\hat\pi}\). Therefore
\begin{align*}
  \EE[r_s(X_{s-1})\mathbf 1[T_{\mathsf G}\ge s]]
  &=
  \PP_{q_0}^{\hat\pi}(T_{\mathsf G}\ge s)
  \int r_s(x)\,O_s^{\hat\pi}(\dd x) \\
  &\le
  2\eps_t\,\PP_{q_0}^{\hat\pi}(T_{\mathsf G}\ge s)
  O_s^{\hat\pi}(\Delta_t\le 2\eps_t) \\
  &\le
  2C_\Delta\,\PP_{q_0}^{\hat\pi}(T_{\mathsf G}\ge s)
  (2\eps_t)^{\beta+1}.
\end{align*}
Summing over \(s=1,\dots,B\) and using the same telescoping argument as
in \Cref{thm:app-occupancy} proves the theorem.
\end{proof}

\begin{proof}[Full proof of \Cref{thm:main}]
For each \(t=1,\dots,B\), \Cref{thm:app-regression} gives an event
\(\mathcal E_t\) of probability at least \(1-\delta_t\) on which
\[
  \sup_{(x,a)\in\mathcal C_t}
  |\widehat Q_t(x,a)-Q_t^*(x,a)|
  \le \eps_t,
\]
with \(\eps_t\) exactly as defined in \cref{eq:main-eps}. By a union
bound, the event
\[
  \mathcal E:=\bigcap_{t=1}^B \mathcal E_t
\]
has probability at least \(1-\sum_{t=1}^B \delta_t\). On \(\mathcal E\),
\Cref{thm:app-occupancy} implies \cref{eq:main-occ}. If the margin
condition holds, then \Cref{thm:app-margin} implies
\cref{eq:main-margin}. This proves the theorem.
\end{proof}

\begin{proof}[Proof of \Cref{cor:conditional}]
If \(\eps_t\le \bar\eps_B\) for all \(t\), then \cref{eq:main-occ}
immediately gives
\[
  \mathrm{SP}_B^{\hat\pi}(q_0)
  \ge
  \mathrm{SP}_B^*(q_0)
  -
  2\bar\eps_B\sum_{s=1}^B
  \PP_{q_0}^{\hat\pi}(T_{\mathsf G}\ge s).
\]
The tail-sum identity yields
\[
  \sum_{s=1}^B
  \PP_{q_0}^{\hat\pi}(T_{\mathsf G}\ge s)
  =
  \mathrm{Len}_B^{\hat\pi}(q_0),
\]
which proves the first claim. For the conditional statement,
\Cref{prop:faithful} implies
\[
  \mathrm{SP}_B^*(q_0)=q_0(\mathrm{Prov}_B=1).
\]
Also, if \(\mathrm{Prov}_B(x)=0\), then \(V_B^{\hat\pi}(x)=0\). Hence
\[
  \mathrm{SP}_B^{\hat\pi}(q_0)
  =
  q_0(\mathrm{Prov}_B=1)\,
  \EE_{q_0}\!\left[
    V_B^{\hat\pi}(X)\mid \mathrm{Prov}_B(X)=1
  \right].
\]
Combining this identity with the first bound yields the result.
\end{proof}

\section{Variable Number of Goals and Truncation}
\label{app:variable-goals}

This appendix expands the variable-goal and compactness discussion in
\cref{sec:setup,sec:implications}.

The measure-valued representation naturally handles a variable number of
open goals, but a globally compact state space may require a mass cap.
Let \(M_{\le W}(K)\) be the set of finite positive measures on a compact
goal-embedding space \(K\) with total mass at most \(W\). Define the
overflow event
\[
  \mathsf{Overflow}
  :=
  \left\{
    \max_{t\le B} X_t(K)>W
  \right\}.
\]
The cap \(W\) should not be read as a claim that real proof states have a
uniformly bounded number of subgoals. It is only a technical device for
moving from a high-probability region of the proof-search process to a
globally compact surrogate. The excluded trajectories are exactly those
on which the active goal multiset becomes too large under the policy and
budget being analyzed.

\begin{proposition}[High-probability truncation]
\label{prop:overflow}
Let \(V_{B,\le W}^\pi\) denote the success probability in the truncated
system that sends overflow states to an absorbing failure state. If
\(\PP^\pi_x(\mathsf{Overflow})\le \delta_W\), then
\[
  V_B^\pi(x)\ge V_{B,\le W}^\pi(x)-\delta_W.
\]
\end{proposition}

\begin{proof}
Couple the true and truncated processes until the first overflow time.
They coincide on the complement of \(\mathsf{Overflow}\). Since the
truncated process only removes successful trajectories after overflow,
the difference in success probability is at most the overflow
probability.
\end{proof}

This is why the main text only assumes compact concentration sets:
global compactness can be recovered by truncation plus an additive tail
term. In applications, \(W\) can be chosen from empirical proof traces or
from a conservative planner budget; the theorem then reports the price of
ignoring wider proof states through the single tail probability
\(\delta_W\).

\section{Representation Learning Through Certificate Tightness}
\label{app:representation}

This appendix expands the Bellman-certificate discussion in
\cref{sec:implications}.

The theorem suggests a concrete objective for representation learning.
Let \(\Phi\) be a class of embeddings \(\phi\) that map formal goals into a
metric space \(K_\phi\). Each \(\phi\) induces a state representation, a
concentration set \(\mathcal C_{\phi,t}\), and hypothesis classes
\(\calH_{\phi,t}\). One can then choose \(\phi\) to balance
approximation and geometry:
\[
  \min_{\phi\in\Phi}
  \EE_{q_0}\!\left[
    U_{B,\phi}^*(X_0)-L_{B,\phi}^*(X_0)
  \right]
  +
  \lambda\,\mathsf{Comp}(\phi),
\]
where \(U_{B,\phi}^*\) and \(L_{B,\phi}^*\) are the tightest
Bellman-based upper and lower certificates available inside a restricted
function class, and \(\mathsf{Comp}(\phi)\) penalizes statistically
unfriendly geometries such as large covering dimension or extreme
Lipschitz distortion.

This viewpoint formalizes a common intuition: a good representation is
not merely predictive, but one under which provability certificates are
tight and statistically estimable. The certificate gap
\(U_{B,\phi}^*(X_0)-L_{B,\phi}^*(X_0)\) measures how much uncertainty
remains about finite-budget provability after restricting attention to
the surrogate family induced by \(\phi\). The complexity penalty then
prevents choosing a representation that makes certificates expressive but
statistically impossible to estimate from verifier-backed traces. This
criterion is deliberately aligned with \cref{thm:main}: the same
representation should reduce approximation error, improve local coverage,
and preserve the action margins that make greedy proving stable.

\section{Scaling-Law Interpretation}
\label{app:scaling}

This appendix expands the scaling-law viewpoint in
\cref{sec:implications,sec:easy-hard}.

Ignoring logarithmic factors and approximation error, the depth-wise
statistical error behaves like
\[
  \eps_t \approx n_t^{-1/(d_t+2)}.
\]
Plugging this into \cref{eq:main-occ} gives a schematic sample-complexity
relationship
\[
  \mathrm{SP}_B^*(q_0)-\mathrm{SP}_B^{\hat\pi}(q_0)
  \lesssim
  \sum_{s=1}^B
  \PP_{q_0}^{\hat\pi}(T_{\mathsf G}\ge s)\,
  n_{B-s+1}^{-1/(d_{B-s+1}+2)}.
\]
Under a margin condition, the exponent improves to
\((\beta+1)/(d_t+2)\).

Three qualitative scaling levers become explicit. First, if
decomposition or lemma introduction reduces the average proof length,
the probabilities \(\PP_{q_0}^{\hat\pi}(T_{\mathsf G}\ge s)\) decay
earlier and the sum has fewer large terms. Second, if retrieval or representation learning
localizes the relevant state-action pairs to a lower-dimensional region,
then the effective dimension \(d_t\) decreases and the rate in \(n_t\)
improves. Third, if verifier feedback and better scoring make near-ties
rarer, the margin exponent \(\beta\) increases and the fast-rate regime
becomes more favorable. These levers correspond to different engineering
interventions, but the bound puts them on the same scale: each reduces
the value-estimation error accumulated along the learned prover's
occupied trajectory. This is the statistical sense in which agentic
components can produce large practical gains without changing worst-case
logical hardness.

\section{Extended Related Work}
\label{app:literature}

This appendix extends the related-work discussion in \cref{sec:related}.

\paragraph{Before the LLM era: tactic and premise guidance.}
Long before current reasoning models, theorem proving research already
recognized two bottlenecks: choosing useful premises from a large
library and choosing the next proof step so that symbolic search does
not explode combinatorially. TacticToe
\citep{Gauthier2017tactictoe,Gauthier2021} is representative: it learns
tactic guidance from existing proofs and combines those predictions with
Monte Carlo tree search. HOList and DeepHOL
\citep{Bansal2019holist} established a large-scale environment for
higher-order theorem proving and made reinforcement-learning and
imitation-learning approaches easier to compare. These systems already
fit the MDP perspective quite naturally: proof states evolve under
discrete actions, learning affects search bias, and evaluation is
budgeted success rather than logical completeness.

\paragraph{Datasets and environments.}
Infrastructure changed the field by making proof interaction
reproducible. LeanDojo \citep{yang2023leandojo} exposed Lean proof
states, tactics, and premise annotations through a programmatic API,
which made retrieval-augmented and verifier-interactive learning much
more practical. miniF2F \citep{zheng2022miniff} provided a cross-system
benchmark built around olympiad-style mathematics and became a common
testing ground for neural and LLM-based formal provers. This benchmark
orientation is relevant to our theory because it sharpened the need for
a distributional notion of success rather than a purely worst-case one.

\paragraph{Neural language models for formal proof.}
GPT-f \citep{Polu2020gen} was an early signal that generative language
models could meaningfully contribute to formal proof search. The shift
from fixed handcrafted features to learned sequence or graph
representations changed both training and inference. Rather than only
predicting the next rule locally, models could propose longer fragments,
re-rank retrieved lemmas, and use verifier feedback to refine failed
attempts. This made theorem proving look less like conventional search
with a static heuristic and more like a dynamic loop of proposal,
verification, and revision.

\paragraph{Agentic provers and the decomposition trend.}
Modern systems increasingly treat theorem proving as a structured
workflow with multiple interacting components. DSP
\citep{jiang2023draft} uses informal proof sketches to steer formal
search toward easier subproblems. DeepSeek-Prover and
DeepSeek-Prover-V2 \citep{xin2024deepseekprover,Ren2025deepseekprover-v2}
emphasize training data generation, recursive decomposition, and RL for
Lean~4 proving. Prover Agent \citep{baba2025proveragent} explicitly
separates roles such as informal reasoning and formal proving. Seed
Prover \citep{bytedance2025seedprover} centers lemma-style whole-proof
reasoning with iterative refinement using Lean feedback. Hilbert
\citep{varambally2025hilbert} and Aristotle
\citep{harmonic2025aristotle} also combine informal reasoning, explicit
intermediate lemmas, and proof-assistant checking. The common pattern is
that decomposition and lemma reuse are no longer auxiliary heuristics;
they are first-class levers for changing the effective proof geometry.

\paragraph{Test-time scaling and verifiers.}
Chain-of-thought prompting \citep{kojima2022cot} and subsequent theory
\citep{feng2023cot-theory,phan2023cot} highlighted that inference-time
computation is part of the effective hypothesis class. In parallel,
recent work has studied test-time scaling more broadly, including how
performance depends on adaptive allocation of compute
\citep{wu2024scaling}. Best-of-\(N\) and verifier-guided inference
analyses \citep{beirami2025bon,setlur2025tts-verifier,botta2025bon}
emphasize that verifier quality can change scaling laws. In theorem
proving, the proof assistant is a nearly perfect verifier for formal
correctness, but agents still rely on imperfect learned signals for
premise usefulness, subgoal quality, or semantic progress. Our theory is
meant to capture this interaction between perfect syntactic checking and
imperfect statistical scoring.

\paragraph{Interactive learning and computational constraints.}
From a broader ML perspective, theorem proving sits inside a long line
of work where learning is driven by interaction rather than passive
i.i.d.\ examples. Angluin's \(L^\ast\) algorithm
\citep{angluin1987Lstar} is a canonical example: the learner uses
membership and equivalence queries to identify a target automaton.
Recent work such as epiplexity \citep{finzi2026epiplexity} and surveys
on formal languages and transformer expressivity
\citep{strobl2024formal-survey} further emphasize that computational
constraints matter when translating information into performance. Our
position is that formal theorem proving requires both viewpoints at
once: proof-theoretic structure determines what short successful
trajectories exist, while statistical learning determines whether an
agent can identify and exploit those trajectories from realistic data.

\end{document}